\documentclass[10pt]{article} 
\usepackage[preprint]{tmlr}

\usepackage{amsmath,amsfonts,bm}

\def\eqref#1{equation~\ref{#1}}

\def\1{\bm{1}}

\DeclareMathAlphabet{\mathsfit}{\encodingdefault}{\sfdefault}{m}{sl}
\SetMathAlphabet{\mathsfit}{bold}{\encodingdefault}{\sfdefault}{bx}{n}

\newcommand{\Var}{\mathrm{Var}}

\newcommand{\Cov}{\mathrm{Cov}}

\usepackage{amsmath}
\usepackage{amsfonts}
\usepackage{amssymb}
\usepackage{amsthm}
\usepackage{booktabs}
\usepackage{mathtools}
\usepackage{array}
\usepackage{enumitem}
\usepackage{hyperref}
\hypersetup{colorlinks=true, linkcolor=blue!50!black, urlcolor=blue!50!black, citecolor=blue!50!black}
\usepackage{url}
\usepackage{tikz}

\title{Learning to Control Coupled-Dynamics Environments with Joint Markov Decision Processes}

\author{\name Ege C. Kaya \email kayae@purdue.edu \\
      \addr Elmore Family School of Electrical and Computer Engineering\\
      Purdue University
      \AND
      \name Aliasghar Pourghani \email apourgha@purdue.edu \\
      \addr Elmore Family School of Electrical and Computer Engineering\\
      Purdue University
      \AND
      \name Mahsa Ghasemi \email mahsa@purdue.edu \\
      \addr Elmore Family School of Electrical and Computer Engineering\\
      Purdue University
      \AND
     \name Vijay Gupta \email gupta869@purdue.edu\\
      \addr Elmore Family School of Electrical and Computer Engineering\\
      Purdue University
      \AND
      \name Abolfazl Hashemi \email abolfazl@purdue.edu \\
      \addr Elmore Family School of Electrical and Computer Engineering\\
      Purdue University}

\newtheorem{definition}{Definition}
\newtheorem{condition}{Condition}
\newtheorem{theorem}{Theorem}

\newtheorem{proposition}{Proposition}
\newtheorem{corollary}{Corollary}[theorem]

\newcommand{\cS}{\mathcal S}
\newcommand{\cA}{\mathcal A}
\newcommand{\cJ}{\mathcal J}
\newcommand{\cM}{\mathcal M}
\newcommand{\cP}{\mathcal P}
\newcommand{\cT}{\mathcal T}
\newcommand{\cX}{\mathcal X}
\newcommand{\cH}{\mathcal H}
\newcommand{\cE}{\mathcal E}

\newcommand{\cB}{\mathcal B}
\newcommand{\bE}{\mathbb E}
\newcommand{\bR}{\mathbb R}
\newcommand{\Law}{\operatorname{Law}}

\def\month{MM}  
\def\year{YYYY} 
\def\openreview{\url{https://openreview.net/forum?id=XXXX}} 

\begin{document}

\maketitle

\begin{abstract}
Coupled-dynamics environments expose the one-step outcomes that would follow from several possible counterfactual actions under a common realization of exogenous randomness. The ordinary Markov decision process formalism allows one to reason about the marginal law of each action but discards dependence across these counterfactual outcomes. The Joint Markov decision process (JMDP) formalism preserves that dependence. Prior work established the formalism and solved the fixed-policy joint moment evaluation problem in JMDPs. This paper develops optimal-control methods. We define a nonparametric distributional Bellman optimality operator for JMDPs, and prove that when the induced marginal MDP has a unique optimal policy, its iterates converge in Wasserstein distance to the optimal joint return law. For the first two moments, we establish convergence under a weaker condition that permits several mean-optimal actions as long as their tie resolutions share a second-moment fixed point. We also derive sampled targets for neural approximation. 
\end{abstract}

\section{Introduction}

In an ordinary Markov decision process (MDP), the agent observes a single sampled outcome at a time: the state transition and the reward produced by the action it executes. This abstraction is sufficient for standard expected-return control. However, in so-called coupled-dynamics environments, where the one-step outcomes of different actions at the same state are generated by some shared randomness, such a model can discard rich structure. In a coupled-dynamics environment, the one-step outcomes of several counterfactual actions can be queried under the same realization of exogenous randomness through a \emph{joint sample transition model} \citep{kaya2026joint}, even though the agent still commits to only one action. The coordinate marginals of these outcomes define an ordinary MDP, but marginalization discards information relating the dependence among counterfactual rewards and transitions.

\citet{kaya2026joint} proposed joint MDPs (JMDPs) as a formalism that can retain such information. A JMDP generates an ordinary MDP through marginalization. Thus, executed policies and their expected returns remain compatible with the ordinary MDP formulation. However, in addition, a JMDP also maintains a coupled one-step outcome table $((R^a, S'^a))_{a \in \cA}$ specifying counterfactual rewards and successor states for every action permissible at a given state $s$. This makes comparative cross-action distributional quantities well-defined, including covariances, action-gap uncertainty, probabilities of superiority, and one can then derive the one-step coupling regime, fixed-policy joint return laws, and Bellman operators for joint moments. However, in spite of these modeling advantages provided by the JMDP formalism, the optimal control problem for JMDPs remains open. Extending the prior fixed-policy results to control requires accounting for a greedy selector that changes with the current mean estimates. Distinct mean-optimal policies can also induce different joint return laws even when they have the same value function.

We study control at the distributional and moment levels. In particular, we show that the nonparametric Bellman operator converges under a unique mean-optimal policy because ordinary value iteration eventually fixes the greedy selector, after which the updates follow a contractive fixed-policy operator. We then show that the first two moments converge under a weaker condition that permits optimal ties with a common second-moment fixed point. We also derive sampled moment targets for function approximation, with each cross-moment continuation evaluated at both successor state--action pairs. The experiments examine exact distributional recursion, finite-dimensional moment updates, covariance-dependent allocation, and neural estimation from coupled samples.

\noindent\textbf{Contributions.}
\begin{itemize}[leftmargin=*]
\item We define a nonparametric JMDP Bellman optimality operator for joint return laws.
\item We prove convergence of its risk-neutral iterates under a unique optimal policy, and convergence of second moments under a weaker common fixed-point condition.
\item We give sampled first- and second-moment targets whose cross term depends on both successor state--action pairs.
\item We validate the distributional and moment recursions in exact tabular problems and show in a continuous-state experiment that coupled branch samples identify joint quantities lost by marginal or shuffled samples.
\end{itemize}

\section{Related work}

\noindent\textbf{Distributional reinforcement learning (DRL).} Early return-distribution methods estimated parametric or nonparametric return laws from distributional Bellman equations \citep{morimura,morimura-parametric}. Modern deep distributional RL was revived by categorical value distributions \citep{c51}, followed by quantile and implicit quantile representations \citep{iqn,qr-dqn,FQF}. These methods are powerful marginal estimators of the return distribution for each action, and risk-aware policies can be obtained from the estimated marginal quantiles \citep{iqn}. They do not, however, by themselves identify the joint law of counterfactual action returns at a state. This joint law is of importance whenever the statistic of interest is comparative rather than marginal.

\noindent\textbf{Joint return information.} The JMDP framework of \citet{kaya2026joint} is the starting point for the current work, which formalizes coupled-dynamics environments through joint sample transition models, defines joint return laws, and proves fixed-policy joint policy evaluation results for finite-order moments. We take JMDPs as established prior work and study the problem of optimal control. Closely related motivations also appear in work on action-gap random variables and probabilities of superiority \citep{wiltzer2024action}, where the central object is the joint comparison of action returns rather than a collection of marginal distributions.

\noindent\textbf{Risk-sensitive and risk-aware control.} Working solely with expected returns suppresses variability that can matter in safety-critical and allocation problems. Classical alternatives include variance-sensitive criteria and mean--variance MDPs \citep{sobel,mv-mdp,tamar2}, coherent risk measures \citep{artzner1999coherent}, and CVaR optimization \citep{rockafellar2000optimization,cvar-mdp}. These objectives are usually functionals of the return from one executed action or policy and can therefore be evaluated from its marginal return law. However, joint information is needed for a different class of questions. Comparing two actions requires their joint return law, while evaluating the variance of an allocation across actions requires their cross-action covariances.

\noindent\textbf{Multivariate returns and counterfactual reasoning.} Multivariate DRL often refers to vector-valued rewards or multi-objective returns \citep{bellman-gan,multidimensional-reward,wiltzer2024foundations}. In that literature, the components of the return vector are jointly generated by a single executed action. In a JMDP, they instead index the returns associated with mutually exclusive counterfactual actions. Counterfactual methods for sequential decisions instead condition a structural causal model on observed experience to infer outcomes under unexecuted actions, with applications to off-policy evaluation and policy search \citep{oberst2019counterfactual,buesing2019woulda}. Such inferences require a causal mechanism that couples potential outcomes. However, the transition marginals of a regular MDP do not identify that mechanism. Gumbel-max SCMs select one coupling, whereas generalized Gumbel mechanisms learn couplings for specified query criteria and copula and robust approaches characterize or bound the range of compatible counterfactuals \citep{lorberbom2021learning,haugh2023counterfactualanalysisdynamiclatent,lally2025robust}. Counterfactual realizability asks when a joint counterfactual distribution can instead be sampled through an admissible experiment \citep{raghavan2025counterfactual}. JMDPs make an explicit access assumption of this kind: the coupled action-outcome table is part of the model or sampling interface. We therefore study prospective control under a given coupling and not retrospective inference from ordinary trajectories.

\section{Preliminaries}\label{sec:prelim}
\subsection{MDPs and marginal distributional control}

\begin{definition}\label{defn:mdp}
A finite discounted MDP is a tuple $\cM = (\cS, \cA, P, \gamma)$ where $\cS$ and $\cA$ are finite state and action spaces, $P:\cS\times\cA\to \cP([0,1]\times \cS)$ is the joint reward-transition kernel, and $\gamma \in (0, 1)$ is the discount factor. For $(R,S')\sim P(\cdot \mid s,a)$, we write the mean reward as $r(s,a) = \bE[R \mid s,a]$ and write $P_S(\cdot \mid s,a)$ for the state marginal of $P$.
\end{definition}

For a stationary Markov policy $\pi$, the return starting from $(s, a)$ is
\begin{equation}
Z^\pi(s, a) := \sum_{t=0}^\infty \gamma^t R_t,
\end{equation}
with $S_0 = s$, $A_0 = a$, $(R_t,S_{t+1})\sim P(\cdot\mid S_t,A_t)$, and
$A_{t+1}\sim \pi(\cdot\mid S_{t+1})$ for $t\ge 0$. The mean return is the
usual action-value function $Q^\pi(s, a) = \bE[Z^\pi(s, a)]$. The optimal value
$Q^*$ is the fixed point of the expected Bellman optimality operator
$T:\bR^{\cS\times\cA}\to\bR^{\cS\times\cA}$:
\begin{equation}
(TQ)(s,a) := \bE [R \mid s, a] + \gamma \bE_{S' \sim P_S(\cdot \mid s, a)}\Bigl[\max_{b \in \cA} Q(S', b) \Bigr].
\end{equation}

DRL tracks the law $\eta^\pi(s,a) := \Law\bigl(Z^\pi(s, a)\bigr) \in \cP(\bR)$, which satisfies the distributional Bellman equation
\begin{equation}
Z^\pi(s, a) \overset{D}{=} R(s, a) + \gamma Z^\pi(S', A').
\end{equation}
Here, $(R, S')$ is sampled jointly from $P(\cdot \mid s, a)$ and $A' \sim \pi(\cdot \mid S')$. The distributional Bellman optimality operator replaces $A'$ by a greedy action chosen from the mean of the candidate return distribution. Unlike the expected Bellman optimality operator, the distributional Bellman optimality operator is not contractive. Classical distributional control nevertheless converges in the presence of a unique optimal policy because the greedy policy stabilizes after the value estimates are sufficiently accurate \citep[Chapter 7]{DRL-textbook}.

\section{JMDPs and joint moment evaluation}\label{sec:moment}

\begin{definition}
A finite JMDP is a tuple $\cJ = (\cS, \cA, \gamma, J)$, where $\cS$, $\cA$ and $\gamma$ are as in Definition~\ref{defn:mdp} and $J$ is a Markov kernel over full one-step counterfactual outcome tables:
\begin{equation}
J(\cdot \mid s) \in \cP\Bigl(\bigl([0, 1] \times \cS \bigr)^{\lvert \cA \rvert} \Bigr).
\end{equation}

At state $s$, the environment draws $\bigl((R^a, S'^{a})\bigr)_{a \in \cA} \sim J (\cdot \mid s)$. We call the action-indexed coordinate $(R^a,S'^a)$ the one-step branch for action $a$. The agent executes one action $A$ and observes the realized branch $(R^A,S'^A)$. The remaining branches remain counterfactual observations.
\end{definition}

The induced marginal MDP $\cM_\cJ$ has joint reward-transition kernel $P_\cJ(\cdot\mid s,a)$ equal to the $(R^a,S'^a)$ coordinate marginal of $J(\cdot\mid s)$. For a fixed policy $\pi$, the JMDP joint return law is
\begin{equation}
\eta^\pi(s) = \Law\bigl((Z^\pi(s, a))_{a \in \cA}\bigr) \in \cP(\bR^{\lvert \cA\rvert}).
\end{equation}
We use the one-step coupling of \citet{kaya2026joint}. At each recursion depth, branches at the same current state query one shared fresh outcome table. Branches at distinct states query independent fresh tables. The successor branches are then advanced jointly.

The first moment is $\mu^\pi(s, a) := Q^\pi(s,a) = \bE [Z^\pi(s,a)]$. The uncentered second moment for two state-action pairs is
\begin{equation}
M^\pi(s, a, \tilde s, \tilde a) := \bE [Z^\pi(s, a)\, Z^\pi(\tilde s, \tilde a)].
\end{equation}
For fixed $\pi$, the second moment Bellman equation is
\begin{equation}\label{eq:2ndmoment}
M^\pi(s, a, \tilde s, \tilde a) = \bE \bigl[R^a \tilde R^{\tilde a} + \gamma R^a\mu^\pi(\tilde S'^{\tilde a}, \tilde A') + \gamma \tilde R^{\tilde a} \mu^\pi(S'^a, A') + \gamma^2 M^\pi (S'^a, A', \tilde S'^{\tilde a}, \tilde A') \mid S= s, \tilde S = \tilde s \bigr],
\end{equation}
where, for distinct branches, $A' \sim \pi (\cdot \mid S'^a)$ and
$\tilde A' \sim \pi(\cdot \mid \tilde S'^{\tilde a})$ are drawn conditionally
independently given the one-step outcomes and independently of the
continuation returns, while for the same branch, $(s,a) = (\tilde s, \tilde a)$,
the outcome and action draws are shared, i.e., $\tilde A' = A'$. The one-step
outcomes are sampled according to the JMDP coupling when $s=\tilde s$ and
independently from the corresponding coordinate marginals when $s\neq\tilde s$.
In particular, the diagonal entries satisfy
$M^\pi(s,a,s,a) = \bE\bigl[Z^\pi(s,a)^2\bigr]$.

Under standard single-action execution, a stochastic policy samples its action independently of the unobserved outcome table. Its return law is therefore
\begin{equation}
\Law\bigl(Z^\pi(s,A)\bigr) = \sum_{a\in\cA}\pi(a\mid s)\,\Law\bigl(Z^\pi(s,a)\bigr).
\end{equation}
Future policy randomization is already included in each coordinate law. The coordinate marginals determine the realized return law under single-action execution. Cross-action dependence governs counterfactual comparisons and simultaneous allocations.

\section{Risk-neutral JMDP optimality}\label{sec:risk-neutral}

The fixed-policy theory of \citet{kaya2026joint} characterizes finite-order joint moments, and its $n$th-order Bellman operator is indexed by $(\cS\times\cA)^n$, which permits evaluation of each mixed continuation moment at the resulting tuple of successor state--action pairs. A distributional extension analogously requires a joint law for every such tuple. Let $d=|\cA|\ge2$, fix an ordering of the actions, and set $\cH_d=(\cS\times\cA)^d$. We call $h=((s_i,a_i))_{i=1}^d\in\cH_d$ a configuration. This construct records the state and immediate action associated with each coordinate of a $d$-dimensional return vector. The individual pairs $(s_i,a_i)$ label scalar coordinates, while the law indexed by the full configuration is their joint distribution. The usual JMDP action-return vector at $s$ corresponds to $h_s=((s,a))_{a\in\cA}$, where all coordinates begin at $s$ and the actions enumerate $\cA$. Under a deterministic policy $\pi$, one backup maps $h_s$ to the random successor configuration $h'=((S'^a,\pi(S'^a)))_{a\in\cA}$. Since the states $S'^a$ may differ, $h'$ need not equal $h_x$ for any $x\in\cS$. We therefore define the distributional operator on laws indexed by $\cH_d$, with the original state-indexed JMDP law recovered at $h_s$.

For a metric space $\cX$, let $\cP_p(\cX)$ denote the probability measures with finite $p$th moment. We consider families
\begin{equation}
\eta\in\cP_p(\bR^d)^{\cH_d}
\end{equation}
that assign a joint law to every configuration. The measure $\eta(h)$ specifies the $d$ scalar marginals and their dependence. The family is marginally consistent when the scalar law attached to a state--action pair is independent of the configuration in which that pair appears. Writing $\operatorname{pr}_i$ for the $i$th coordinate projection, this condition is equivalently stated as
\begin{equation}
(\operatorname{pr}_i)_\#\eta(h)=(\operatorname{pr}_j)_\#\eta(\tilde h)\quad\text{whenever}\quad h_i=\tilde h_j.
\label{eq:marginal-consistency}
\end{equation}
Let $\cE_p$ be the set of marginally consistent families, i.e., families in which every occurrence of a given state--action pair $(s,a)$ has the same scalar marginal, denoted by $\eta_{s,a}$. We fix a deterministic policy $\pi$ and a configuration $h=((s_i,a_i))_{i=1}^d\in\cH_d$. To define $\cT^\pi\eta(h)$, we sample one outcome table from $J(\cdot\mid\bar s)$ for every distinct state $\bar s$ among $s_1,\ldots,s_d$, using independent tables at distinct states. Let $(R_i,S'_i)$ be the coordinate for action $a_i$ in the table sampled at $s_i$, and set
\begin{equation}
h'=((S'_i,\pi(S'_i)))_{i=1}^d,\qquad (\cT^\pi\eta)(h)=\Law\bigl((R_i+\gamma Z'_i)_{i=1}^d\bigr),\qquad Z'\sim\eta(h').
\label{eq:configuration-operator}
\end{equation}
If coordinates that currently occupy different states reach the same successor state $x$, their entries in $h'$ both have state component $x$. The following backup therefore samples one table at $x$ and both coordinates read their outcomes from that table. Coordinates whose successor states remain different continue with independent tables. Conditional on $h'$, the vector $Z'$ is independent of the sampled one-step tables and is drawn jointly from $\eta(h')$. The fixed-policy branch process from Section~\ref{sec:moment} induces a family $\eta^\pi_{\cH}\in\cE_p$ satisfying $\cT^\pi\eta^\pi_{\cH}=\eta^\pi_{\cH}$, with $\eta^\pi_{\cH}(h_s)$ equal to the JMDP joint return law $\eta^\pi(s)$.

For $\xi,\upsilon\in\cP_p(\bR^d)$, let $\Pi(\xi,\upsilon)$ be their set of couplings. The ordinary $p$-Wasserstein distance with ground metric $\|x-y\|_\infty=\max_i|x_i-y_i|$ is
\begin{equation}
W_p(\xi,\upsilon)=\inf_{\lambda\in\Pi(\xi,\upsilon)}\left(\bE_{(X,Y)\sim\lambda}\left[\|X-Y\|_\infty^p\right]\right)^{1/p}.
\label{eq:wasserstein}
\end{equation}

On $\cE_p$, take the largest Wasserstein distance across configurations:
\begin{equation}
D_p(\eta,\zeta)=\max_{h\in\cH_d}W_p\bigl(\eta(h),\zeta(h)\bigr).
\label{eq:configuration-metric}
\end{equation}

\begin{proposition}\label{prop:contraction}
Every deterministic stationary policy $\pi$ induces a map $\cT^\pi:\cE_p\to\cE_p$. For all $\eta,\zeta\in\cE_p$,
\begin{equation}
D_p(\cT^\pi\eta,\cT^\pi\zeta)\le\gamma D_p(\eta,\zeta).
\end{equation}
\end{proposition}
\begin{proof}
Bounded rewards and the finite $p$th moments of the continuation laws imply that every output law has finite $p$th moment. Consider two coordinates with $h_i=\tilde h_j=(s,a)$. Under their respective updates, both coordinate marginals are the law of $R^a+\gamma Z$, where $((R^c,S'^c))_{c\in\cA}\sim J(\cdot\mid s)$ and, conditional on $S'^a$, $Z\sim\eta_{S'^a,\pi(S'^a)}$. This law depends only on $(s,a)$, so $\cT^\pi\eta$ satisfies \eqref{eq:marginal-consistency}.

Now fix $h\in\cH_d$ and couple the two updates using the same one-step tables in every current-state group. For each possible successor configuration $h'$, choose an optimal coupling $(U,V)$ of $\eta(h')$ and $\zeta(h')$. Such a coupling exists for the continuous cost $\|\cdot\|_\infty^p$ on $\bR^d$. Since $\cH_d$ is finite, the couplings can be fixed for every possible $h'$ and drawn independently of the rewards conditional on $h'$. The coupled outputs are $Y=R+\gamma U$ and $\widetilde Y=R+\gamma V$, hence
\begin{equation}
W_p\bigl((\cT^\pi\eta)(h),(\cT^\pi\zeta)(h)\bigr)^p\le \bE\left[\|Y-\widetilde Y\|_\infty^p\right] =\gamma^p\bE_{h'}\left[W_p\bigl(\eta(h'),\zeta(h')\bigr)^p\right]\le\gamma^pD_p(\eta,\zeta)^p.
\end{equation}
The first inequality evaluates the Wasserstein infimum at this coupling. Taking $p$th roots and maximizing over $h$ proves the result.
\end{proof}

Proposition~\ref{prop:contraction} assumes a fixed continuation policy. Under greedy continuation, equal-mean actions may carry different return laws, so a tie can change the law selected in the next backup. Let us define
\begin{equation}
Q_\eta(s,a)=\bE_{Z\sim\eta_{s,a}}[Z].
\end{equation}
For the fixed-policy family, $Q_{\eta^\pi_{\cH}}(s,a)=Q^\pi(s,a)$. Fix a deterministic tie-breaking rule, let
\begin{equation}
[G(\eta)](s)\in\arg\max_{a\in\cA}Q_\eta(s,a),
\end{equation}
and define $\cT_G\eta:=\cT^{G(\eta)}\eta$. Under the next condition, the greedy policy eventually stops changing and stabilizes at a fixed policy.
\begin{condition}[Unique risk-neutral optimal policy]\label{cond:unique-opt}
The induced marginal MDP $\cM_\cJ$ has a unique optimal action $\pi^*(s)$ at every state. Since $\cS$ and $\cA$ are finite, the action gap
\begin{equation}
\Delta^* = \min_{s \in \cS}\Bigl(Q^*\bigl(s, \pi^*(s)\bigr) - \max_{a \ne \pi^*(s)} Q^*(s, a) \Bigr) > 0.
\end{equation}
\end{condition}
\begin{theorem}\label{thm:unique-opt}
Under Condition~\ref{cond:unique-opt}, for any $p\ge1$ and $\eta_0\in\cE_p$, the iterates $\eta_{k+1}=\cT_G\eta_k$ satisfy
\begin{equation}
D_p(\eta_k,\eta^{\pi^*}_{\cH})\to0.
\end{equation}
In particular, $W_p(\eta_k(h_s),\eta^{\pi^*}_{\cH}(h_s))\to0$ for every $s$.
\end{theorem}
\begin{proof}
Let $q_k=Q_{\eta_k}$. Taking any coordinate marginal in \eqref{eq:configuration-operator} and using \eqref{eq:marginal-consistency} gives
\begin{equation}
q_{k+1}(s,a)=\bE_{((R^c,S'^c))_{c\in\cA}\sim J(\cdot\mid s)}\left[R^a+\gamma q_k\bigl(S'^a,G(\eta_k)(S'^a)\bigr)\right]=(Tq_k)(s,a).
\end{equation}
Hence $\|q_k-Q^*\|_\infty\le\gamma^k\|q_0-Q^*\|_\infty$. Choose $K$ so that $\|q_k-Q^*\|_\infty<\Delta^*/2$ for every $k\ge K$. For any $a\ne\pi^*(s)$,
\begin{equation}
q_k\bigl(s, \pi^*(s)\bigr) - q_k(s, a) \ge Q^*\bigl(s, \pi^*(s)\bigr) - Q^*(s, a) - 2 \lVert q_k - Q^*\rVert_\infty >0,
\end{equation}
which shows that $G(\eta_k)=\pi^*$ for all $k\ge K$. The Wasserstein space $\cP_p(\bR^d)$ is complete. Since $\cH_d$ is finite and marginal consistency is preserved under limits, $(\cE_p,D_p)$ is also complete. Proposition~\ref{prop:contraction} then gives the unique fixed family $\eta^{\pi^*}_{\cH}$. For every $n\ge0$,
\begin{equation}
D_p(\eta_{K+n},\eta^{\pi^*}_{\cH})\le\gamma^nD_p(\eta_K,\eta^{\pi^*}_{\cH})\to0.
\end{equation}
\end{proof}

\begin{corollary}\label{cor:moment2}
With $p=2$, the coordinate means and pairwise second moments of $\eta_k(h_s)$ converge uniformly in $s$ to those of the optimal joint return law.
\end{corollary}

\begin{proof}
Fix $s\in\cS$. Choose couplings $(Z_k,Z)$ of $\eta_k(h_s)$ and $\eta^{\pi^*}_{\cH}(h_s)$ such that
\begin{equation}
\bE\left[\|Z_k-Z\|_\infty^2\right]\to0.
\end{equation}
For each coordinate $a$,
\begin{equation}
\left|\bE[Z_{k,a}]-\bE[Z_a]\right|\le\bE\left[|Z_{k,a}-Z_a|\right]\le\left(\bE\left[\|Z_k-Z\|_\infty^2\right]\right)^{1/2}\to0.
\end{equation}
For pairwise second moments,
\begin{equation}
|Z_{k,a}Z_{k,b}-Z_aZ_b|\le|Z_{k,a}-Z_a|\,|Z_{k,b}|+|Z_a|\,|Z_{k,b}-Z_b|.
\end{equation}
Finally, by Cauchy--Schwarz,
\begin{equation}
\bE\left[|Z_{k,a}Z_{k,b}-Z_aZ_b|\right]\le\left(\bE\left[\|Z_k-Z\|_\infty^2\right]\right)^{1/2}\cdot\left\{\left(\bE\left[\|Z_k\|_\infty^2\right]\right)^{1/2}+\left(\bE\left[\|Z\|_\infty^2\right]\right)^{1/2}\right\}.
\end{equation}
The factor in braces remains bounded because $W_2$ convergence controls second moments. Taking maxima over the finite state and action sets proves the claim.
\end{proof}

The same argument gives mixed moments up to any fixed integer degree $n$ by applying Theorem~\ref{thm:unique-opt} with $p=n$. We state the result for order two as a representative and because that is the case used in the following section.

\subsection{A weaker condition for moment control}

Ordinary value iteration gives convergence of the means without Condition~\ref{cond:unique-opt} because all optimal actions have the same value. Ties can still change the continuation variance and cross-action covariance, so convergence of second moments requires an additional condition. Let $\cA^*(s)=\arg\max_{a\in\cA}Q^*(s,a)$. For any selector $\sigma$ satisfying $\sigma(s)\in\cA^*(s)$, define
\begin{equation}\label{eq:optimal-2ndmoment}
\begin{aligned}
\bigl[\cB^*_\sigma M\bigr](s,a,\tilde s,\tilde a)=\bE\Bigl[&
R^a\tilde R^{\tilde a}+\gamma R^a Q^*\bigl(\tilde S'^{\tilde a},\sigma(\tilde S'^{\tilde a})\bigr)+\gamma\tilde R^{\tilde a}Q^*\bigl(S'^a,\sigma(S'^a)\bigr)\\
&+\gamma^2M\bigl(S'^a,\sigma(S'^a),\tilde S'^{\tilde a},\sigma(\tilde S'^{\tilde a})\bigr)\,\Bigm|\,S=s,\tilde S=\tilde s\Bigr].
\end{aligned}
\end{equation}
The expectation uses the one-step coupling from \eqref{eq:2ndmoment}. Each $\cB^*_\sigma$ is a $\gamma^2$-contraction in the sup norm, as shown in the proof of Proposition~\ref{prop:2nd-indistin}, and therefore has a unique fixed point $M^\sigma$.
\begin{condition}[Common second-moment fixed point across optimal selectors]\label{cond:weaker}
For every pair of optimal selectors $\sigma$ and $\sigma'$, the fixed points agree:
\begin{equation}
M^\sigma = M^{\sigma'} =: M^*.
\end{equation}
\end{condition}
\begin{proposition}\label{prop:2nd-indistin}
Assume Condition~\ref{cond:weaker}. Start from any $q_0\in\bR^{\cS\times\cA}$ and $M_0\in\bR^{(\cS\times\cA)^2}$. At each iteration, choose any greedy selector
\begin{equation}
\sigma_k(s)\in\arg\max_{c\in\cA}q_k(s,c).
\end{equation}
Update the mean by ordinary value iteration,
\begin{equation}
q_{k+1}(s,a)=r(s,a)+\gamma\sum_{s'\in\cS}P_S(s'\mid s,a)\max_{c\in\cA}q_k(s',c),
\end{equation}
and update the second moment by
\begin{equation}
\begin{aligned}
M_{k+1}(s,a,\tilde s,\tilde a)=\bE\Bigl[&
R^a\tilde R^{\tilde a}+\gamma R^a q_k\bigl(\tilde S'^{\tilde a},\sigma_k(\tilde S'^{\tilde a})\bigr)+\gamma\tilde R^{\tilde a}q_k\bigl(S'^a,\sigma_k(S'^a)\bigr)\\
&+\gamma^2M_k\bigl(S'^a,\sigma_k(S'^a),\tilde S'^{\tilde a},\sigma_k(\tilde S'^{\tilde a})\bigr)\,\Bigm|\,S=s,\tilde S=\tilde s\Bigr].
\end{aligned}
\end{equation}
The expectation again uses the coupling from \eqref{eq:2ndmoment}. Then $q_k\to Q^*$ and $M_k\to M^*$ as $k\to\infty$, where $M^*$ is the common fixed point in Condition~\ref{cond:weaker}.
\end{proposition}
\begin{proof}
The mean recursion is ordinary value iteration, so $\|q_k-Q^*\|_\infty\to0$. If a strictly suboptimal action exists, define
\begin{equation}
\Delta_{\mathrm{sub}} := \min_{s \in \cS, a \notin \cA^*(s)}\Bigl[\max_{b \in \cA}Q^*(s, b) - Q^*(s, a) \Bigr] > 0.
\end{equation}
Choose $K$ so that $\|q_k-Q^*\|_\infty<\Delta_{\mathrm{sub}}/2$ for $k\ge K$. If every action is optimal, take $K=0$. Then $\sigma_k(s)\in\cA^*(s)$ for all $k\ge K$. Write the update as $M_{k+1}=\cB^*_{\sigma_k}M_k+\varepsilon_k$, where $\varepsilon_k$ replaces the two occurrences of $Q^*$ in \eqref{eq:optimal-2ndmoment} by $q_k$. Since rewards lie in $[0,1]$,
\begin{equation}
\|\varepsilon_k\|_\infty\le2\gamma\|q_k-Q^*\|_\infty\to0.
\end{equation}
For every optimal selector $\sigma$, only the continuation term depends on $M$, giving
\begin{equation}
\|\cB^*_\sigma M-\cB^*_\sigma N\|_\infty\le\gamma^2\|M-N\|_\infty.
\end{equation}
The affine maps may differ, but Condition~\ref{cond:weaker} gives the common fixed point $\cB^*_{\sigma_k}M^*=M^*$. Hence
\begin{equation}
\lVert M_{k+1} - M^*\rVert_\infty \le \gamma^2 \lVert M_k - M^*\rVert_\infty + \lVert \varepsilon_k \rVert_\infty.
\end{equation}
Iterating from step $K$ gives
\begin{equation}
\lVert M_{K +n} - M^*\rVert_\infty \le \gamma^{2n}\lVert M_K - M^*\rVert_\infty + \sum_{i=0}^{n-1}\gamma^{2(n-1-i)}\lVert \varepsilon_{K+i}\rVert_\infty.
\end{equation}
The first term vanishes geometrically. The sum vanishes because $\varepsilon_k\to0$ and $\sum_{j\ge0}\gamma^{2j}<\infty$. Thus $M_k\to M^*$.
\end{proof}

Condition~\ref{cond:weaker} permits several optimal actions when every tie resolution yields the same second-moment fixed point. Differences in the return variances, covariances, or pairwise second moments induced by alternative optimal continuation actions can cause the condition to fail.

\section{Approximate joint moment control}\label{sec:approximate-control}
For large state spaces, we may fit the first two moment equations from sampled branch tables. We reserve $M$ for exact tabular moments and denote their parametric approximation by a mean head $\mu_\theta(s,a)$ and an uncentered second-moment head
\begin{equation}
\Sigma_{\theta,\psi}(s,a,\tilde s,\tilde a)\approx \bE[Z(s,a)Z(\tilde s,\tilde a)].
\end{equation}
The $\theta$ subscript for the second-moment head indicates that the second moment may also depend on the mean parameters. For example, a positive-semidefinite covariance kernel is obtained from
\begin{equation}
\Sigma_{\theta,\psi}(s,a,\tilde s,\tilde a)=\mu_\theta(s,a)\mu_\theta(\tilde s,\tilde a)+\phi_\psi(s,a)^\top\phi_\psi(\tilde s,\tilde a).
\end{equation}

Let $g_{\bar\theta}(s')\in\arg\max_{c\in\cA}\mu_{\bar\theta}(s',c)$, where bars denote fixed target parameters. A mean sample consists of $(s,a)$ and one outcome $(r,s')$ from the corresponding coordinate marginal of $J$, with its target being
\begin{equation}
y_\mu=r+\gamma\mu_{\bar\theta}(s',g_{\bar\theta}(s')).
\end{equation}

Second moments require samples of pairs $(s,a,\tilde s,\tilde a)$. When $s=\tilde s$, we draw one table from $J(\cdot\mid s)$ and use its $a$ and $\tilde a$ coordinates. When $s\ne\tilde s$, we draw independent tables at the two states and write the resulting outcomes as $(r,s')$ and $(\tilde r,\tilde s')$. Then, the sampled uncentered second-moment target is
\begin{equation}\label{eq:second-target}
y_\Sigma=r\tilde r+\gamma r\mu_{\bar\theta}(\tilde s',g_{\bar\theta}(\tilde s'))+\gamma\tilde r\mu_{\bar\theta}(s',g_{\bar\theta}(s'))+\gamma^2\Sigma_{\bar\theta,\bar\psi}\bigl(s',g_{\bar\theta}(s'),\tilde s',g_{\bar\theta}(\tilde s')\bigr).
\end{equation}

Note that the final term depends on both successor states. When $s'\ne\tilde s'$, it evaluates the continuation moment between two different state--action pairs. A covariance block local to either successor would discard this dependence.

Let $\nu_1$ be a sampling distribution over state--action pairs and let $\nu_2$ be a sampling distribution over pairs of state--action pairs. Let $\mathsf K_J(\cdot\mid s,a,\tilde s,\tilde a)$ denote the coupled pair-outcome kernel just described. The population objective is then
\begin{equation}
\mathcal L(\theta,\psi)=\bE_{\substack{(s,a)\sim\nu_1,\\(r,s')\sim P_\cJ(\cdot\mid s,a)}}\left[(\mu_\theta(s,a)-y_\mu)^2\right]+\lambda\bE_{\substack{(s,a,\tilde s,\tilde a)\sim\nu_2,\\(r,s',\tilde r,\tilde s')\sim\mathsf K_J(\cdot\mid s,a,\tilde s,\tilde a)}}\left[(\Sigma_{\theta,\psi}(s,a,\tilde s,\tilde a)-y_\Sigma)^2\right],
\label{eq:population-moment-loss}
\end{equation}
where the hyperparameter $\lambda\ge0$ balances the first- and second-moment terms. Training minimizes a sampled minibatch average of this objective. A queried action set supplies same-state pairs from one table. Propagated branch pairs supply the distinct-state rows used by the continuation target. In the parameterization above, gradients from the second-moment term pass through both $\theta$ and $\psi$, including the product $\mu_\theta(s,a)\mu_\theta(\tilde s,\tilde a)$.

\section{Experiments}\label{sec:experiments}

In this section, we first compare the distributional and moment control recursions with dynamic-programming and Monte Carlo references. We then study how cross-action covariance affects a risk-aware route choice and whether neural models can recover joint return information from sampled outcome tables.

\subsection{Exact distributional control in a coupled-reward chain}

We first use a simple environment in the form of a chain of states with a coupled reward structure. Two actions are available, and $\gamma=0.95$.  The unique-policy chain has five nonterminal states $\{0,\ldots,4\}$ followed by terminal state $5$, actions $\{0,1\}$, and $\gamma=0.95$. Action 0 advances one state and action 1 advances two states, clipped at the terminal state. At each nonterminal state, a shared Bernoulli variable selects reward vector $(1,0)$ or $(0.2,0.8)$ with equal probability. Action 0 is uniquely optimal at every nonterminal state. We run seven exact backups and compare the iterates with ordinary value iteration, fixed-policy distributional evaluation of the optimal law, and exact second-moment evaluation. Figure~\ref{fig:distributional-convergence} shows convergence of the law, means, and covariance blocks.

To contrast this convergent case with the behavior under optimal ties, we use a separate two-state JMDP. The zero-reward start state leads to a terminal choice state. At the choice state, two equiprobable outcome tables give action-0 rewards $0$ and $2$, while action 1 always gives reward $1$. Both actions have mean one but different return laws. In the lower panels of Figure~\ref{fig:distributional-convergence}, we deliberately alternate between the two mean-greedy tie resolutions over 12 backups, and the distributional iterates alternate between the two optimal laws. Figure~\ref{fig:distribution-support} returns to the unique-policy chain and shows its joint law at the start state. The two shared reward outcomes shift the coordinates in opposite directions, while the different successor depths give the coordinates different continuation ranges. This produces the offset bands in the figure.

\begin{figure}[t]
\centering
\includegraphics[width=\linewidth]{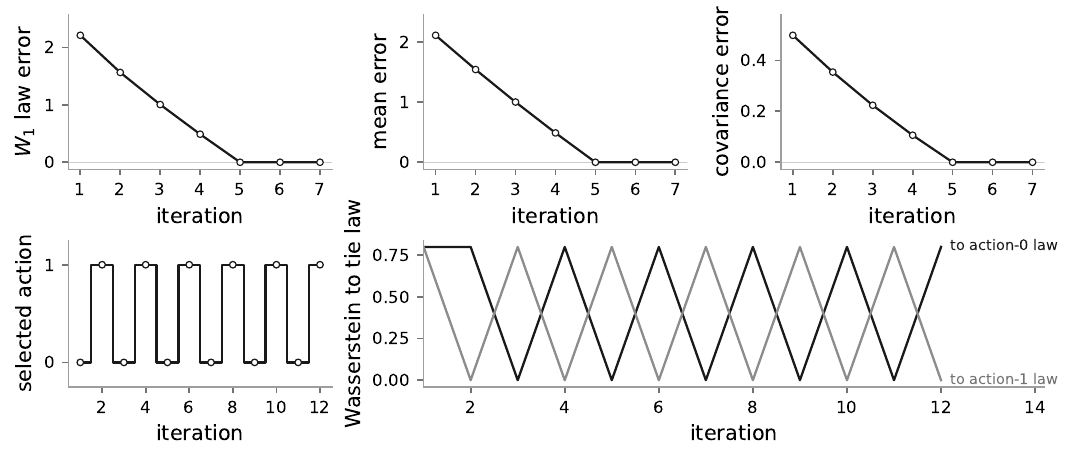}
\caption{Exact tabular distributional control. In the top row, the optimal action is unique at every nonterminal state, and the joint law, mean vector, and covariance blocks converge to their ground-truth values. The bottom panels use an alternating sequence of mean-greedy tie resolutions, which switches between two optimal return laws.}
\label{fig:distributional-convergence}
\end{figure}

\begin{figure}[t]
\centering
\includegraphics[width=0.88\linewidth]{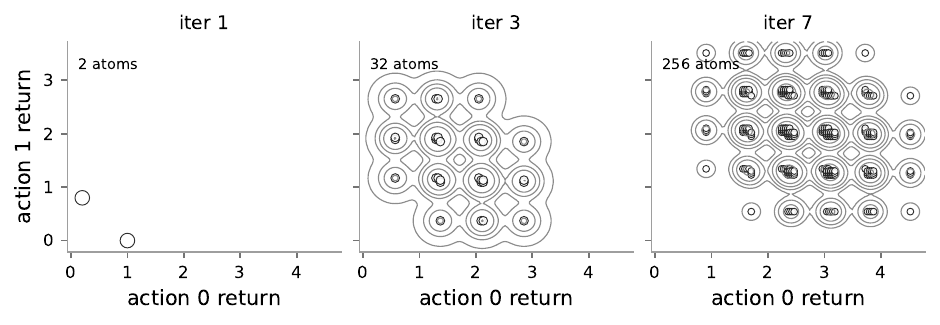}
\caption{Joint return law at the start state of the coupled reward chain. Dots are atoms, i.e., support points with positive probability. Contours smooth the same discrete mass for readability.}
\label{fig:distribution-support}
\end{figure}

\subsection{Joint moment control in a windy gridworld}

The next experiment is a $5\times 5$ windy gridworld with four actions and $\gamma=0.95$. Each state has two outcome tables. With probability $0.75$, every counterfactual action moves normally. With probability $0.25$, a shared gust shifts every successor one column left. Boundary moves are clipped, and entering the lower-right goal gives reward one and terminates. Enumerating the full joint return law requires tracking a growing discrete support over four action-return coordinates. The Bellman equations for $\mu$ and $M$ remain finite-dimensional, so we apply the greedy moment recursions directly. Several states have an exact tie between moving right and moving down. Ordinary value iteration and moment iteration use the same fixed lowest-index tie rule.

Moment iteration stops at residual $10^{-8}$. Means are checked against ordinary value iteration and covariance blocks against a separate fixed-policy JMDP moment evaluation under the resulting greedy policy. A second check uses 8000 jointly propagated Monte Carlo rollouts of horizon 96 at states $0$, $4$, $12$, and $23$. Figure~\ref{fig:gridworld-moments} shows agreement with both references. The goal is the only source of reward, and the shared leftward gust can delay but never hasten reaching it. A gust therefore lowers several counterfactual returns together, while its absence leaves those returns higher, producing the nonnegative correlations shown.

\begin{figure}[t]
    \centering
    \includegraphics[width=\linewidth]{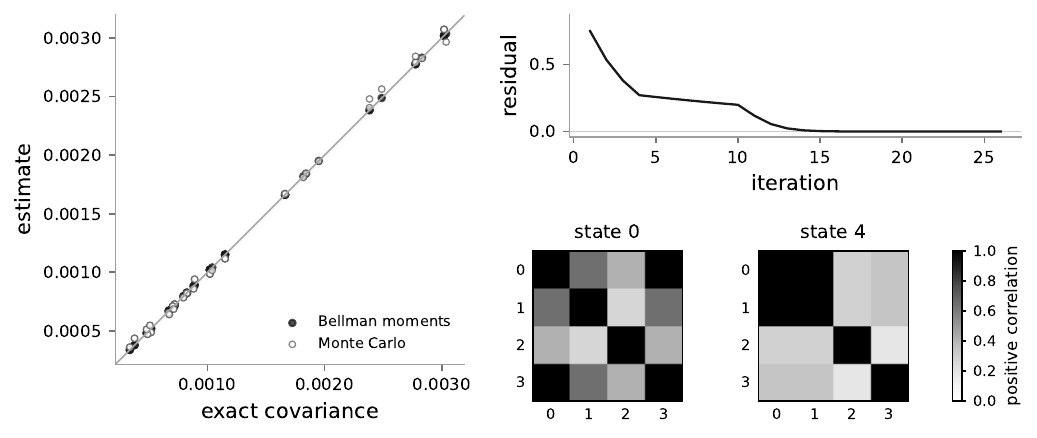}
    \caption{Moment control in a stochastic windy gridworld. The residual decays to zero, the learned covariance blocks match a separate fixed-policy JMDP moment evaluation, and finite-horizon Monte Carlo rollouts give an independent rollout-based comparison. The grayscale matrices show action correlations at representative states.}
    \label{fig:gridworld-moments}
\end{figure}

\subsection{The anti-correlated route choice environment}

In this experiment, a one-decision route-choice JMDP shows how cross-action covariance affects traffic allocation. The available routes are one safe route and two risky ridge routes. The episode terminates after the route outcome, so each immediate reward $R^a$ is also the complete return $Z_a$ for that action. The coupled reward vector $(R^{\mathrm{safe}},R^{\mathrm{left}},R^{\mathrm{right}})$ equals $(1,8,-4)$ with probability $0.40$, $(1,-4,8)$ with probability $0.40$, $(1,8,8)$ with probability $0.10$, and $(1,0,0)$ with probability $0.10$. The safe route is deterministic, and the ridge routes have equal means and marginal variances. A shared wind makes one ridge fail when the other succeeds in most tables, giving the ridge returns negative covariance. The decision maker chooses an allocation $w\in\Delta(\cA)$, where $w_a$ is the fraction of traffic assigned to route $a$. The resulting random return is $w^\top Z$. We choose the allocation in spirit of the mean-risk portfolio optimization problem of \citet{Markowitz} by maximizing
\begin{equation}\label{eq:route-markowitz}
\max_{w\in\Delta(\cA)}\quad w^\top\mu-\beta\sqrt{w^\top Cw},
\end{equation}
where $\mu_a=\bE[Z_a]$ and $C_{ab}=\Cov(Z_a,Z_b)$. The variance term is
\begin{equation}
w^\top Cw=\sum_a w_a^2\Var(Z_a)+2\sum_{a<b}w_aw_b\Cov(Z_a,Z_b).
\end{equation}
We use $\beta=0.5$ and solve \eqref{eq:route-markowitz} on a simplex grid of spacing $0.02$. The negative cross term makes an equal allocation between the ridge routes free of catastrophic loss in this example, although its return remains variable. A diagonal baseline solves the same objective after setting $C_{ab}=0$ for $a\ne b$. The traffic-allocation decision combines all route outcomes from the same wind realization in $w^\top Z$. By contrast, a randomized single-route policy draws $A\sim w$ independently of the table and has catastrophe probability $\sum_aw_ap_a$, where $p_a=\Pr(Z_a<0)$. The variance is determined by the marginal laws and contains no off-diagonal covariance. We also select a single route by lower-tail CVaR at level $0.1$ and by the superiority matrix $H_{ab}=\Pr(Z_a>Z_b)$, using $\arg\max_a\min_{b\ne a}H_{ab}$. Figure~\ref{fig:risk-route} separates these single-route rules from the covariance-sensitive allocation.

\begin{figure}[t]
    \centering
    \includegraphics[width=0.92\linewidth]{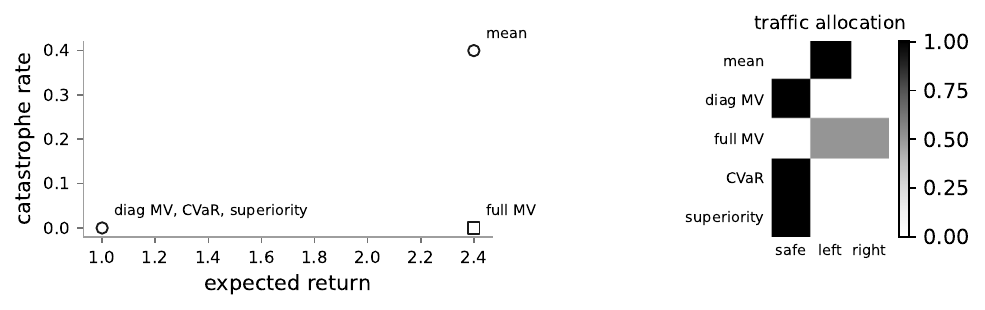}
    \caption{Route choice with coupled failures. Circles denote single-route rules and the square denotes the solution of \eqref{eq:route-markowitz}. ``diag MV'' sets off-diagonal covariances to zero, and ``full MV'' uses the full matrix. The right panel shows the resulting traffic allocation. ``Superiority'' uses $\arg\max_a\min_{b\ne a}\Pr(Z_a>Z_b)$.}
    \label{fig:risk-route}
\end{figure}

\subsection{Neural approximation in a coupled inverted pendulum task}

In this experiment, we test whether neural estimators trained from sampled outcome tables can recover joint moments and a particle approximation of the joint return law in a continuous-state JMDP. The task is a coupled inverted pendulum whose state is angle--velocity $(\vartheta,\omega)$ and whose action force is $a\in\{-1,0,1\}$. Initial states use $\vartheta\sim\mathrm{Unif}[-1.25,1.25]$ and $\omega\sim\mathcal N(0,0.45^2)$. For $\epsilon\sim\mathcal N(0,1)$, the dynamics and reward are
\begin{equation}
\begin{gathered}
\omega'=0.72\omega+0.30a+0.18\epsilon-0.10\vartheta,\qquad \vartheta'=\vartheta+\omega',\\
r=0.75-0.18\vartheta^2-0.06\omega^2-0.035a^2+0.75a\epsilon-0.10\epsilon^2.
\end{gathered}
\end{equation}
The first disturbance is shared by all three action branches. Thereafter, each branch follows a fixed stabilizer that chooses $-1$ when $\vartheta+0.65\omega>0.18$, $1$ when it is below $-0.18$, and $0$ otherwise. Future disturbances are independent across distinct continuous states. Returns have horizon 10 and $\gamma=0.92$. Finite-horizon return regression tests whether a neural model can identify the coupling under this fixed continuation policy.

The moment model predicts means and a symmetric second-moment matrix. The diagonal baseline retains the predicted marginal variances but sets every off-diagonal covariance to zero. We also fit a fixed-size joint particle model. For both architectures, a shuffled baseline independently permutes each action coordinate across tables. This preserves the conditional marginals and removes their coupling. The particle models use a sample estimate of the energy distance \citep{SZEKELY20131249}
\begin{equation}
\mathcal E(P,Q)=2\bE\left[\|X-Y\|_2\right]-\bE\left[\|X-X'\|_2\right]-\bE\left[\|Y-Y'\|_2\right],
\end{equation}
where $X,X'\sim P$ and $Y,Y'\sim Q$ are independent. It is nonnegative and vanishes exactly when $P=Q$, so smaller values indicate closer joint laws. For each of five seeds, models are trained for 20,000 minibatches of 128 states with eight return vectors per state. The moment networks have two 96-unit hidden layers and use mean MSE plus $0.2$ times symmetric second-moment MSE. The particle networks output 32 vectors and use $0.08$ times the energy loss. All networks use Adam \citep{kingma2017adammethodstochasticoptimization} with learning rate $3\times10^{-4}$. Evaluation uses 512 held-out states and 256 Monte Carlo vectors per state, with energy distance evaluated on the first 96 states. Figure~\ref{fig:neural-jmdp-control} shows that the joint moment model has lower covariance and correlation MAE than the diagonal and shuffled baselines. The diagonal baseline cannot recover off-diagonal entries because it sets them to zero. The particle model trained on coupled tables improves on its shuffled counterpart and attains lower covariance error than the moment model in this experiment.

\begin{figure}[t]
    \centering
    \IfFileExists{figures/neural_jmdp_control.pdf}{
        \includegraphics[width=\linewidth]{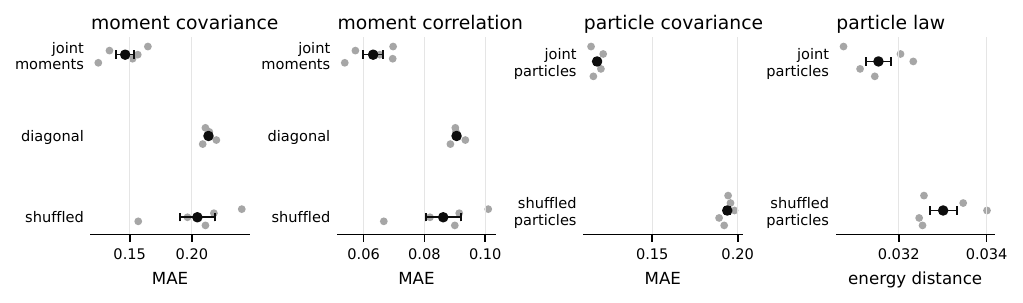}
    }{
        \fbox{\begin{minipage}[c][1.2in][c]{0.88\linewidth}
        \centering Neural JMDP figure pending cluster run.
        \end{minipage}}
    }
    \caption{Neural approximation in the coupled inverted-pendulum JMDP. All panels evaluate neural estimators. The first two show moment models and the last two show particle models. MAE denotes mean absolute error over off-diagonal entries. The covariance panels share a horizontal scale. Energy distance compares the learned and Monte Carlo joint laws, with lower values indicating better agreement. Dots are five seeds and bars show the standard error.}
    \label{fig:neural-jmdp-control}
\end{figure}

\section{Conclusion}
This work extends JMDPs from fixed-policy evaluation to control. The nonparametric distributional iterates converge when the mean-optimal policy is unique, and the first two moments converge under the weaker common-fixed-point condition for optimal ties. The sampled targets retain both successor state--action pairs in every cross moment. JMDP control preserves the mean-optimal policy of the induced MDP and also determines how its counterfactual action returns move together. The tabular and neural experiments support these recursions and show that their joint quantities can be recovered from coupled outcome tables.

\newpage

\bibliography{refs}

@inproceedings{
kaya2026joint,
title={Joint {MDP}s and Reinforcement Learning in Coupled-Dynamics Environments},
author={Ege Can Kaya and Mahsa Ghasemi and Abolfazl Hashemi},
booktitle={Forty-Second Annual Conference on Uncertainty in Artificial Intelligence},
year={2026},
url={https://openreview.net/forum?id=OgERnfBJyc}
}

@InProceedings{c51,
  title = 	 {A Distributional Perspective on Reinforcement Learning},
  author =       {Marc G. Bellemare and Will Dabney and R{\'e}mi Munos},
  booktitle = 	 {Proceedings of the 34th International Conference on Machine Learning},
  pages = 	 {449--458},
  year = 	 {2017},
  editor = 	 {Precup, Doina and Teh, Yee Whye},
  volume = 	 {70},
  series = 	 {Proceedings of Machine Learning Research},
  month = 	 {06--11 Aug},
  publisher =    {PMLR},
  url = 	 {https://proceedings.mlr.press/v70/bellemare17a.html}
}

@inproceedings{cvar-mdp,
author = {Chow, Yinlam and Ghavamzadeh, Mohammad},
title = {Algorithms for {CV}a{R} optimization in {MDP}s},
year = {2014},
publisher = {MIT Press},
address = {Cambridge, MA, USA},
booktitle = {Proceedings of the 28th International Conference on Neural Information Processing Systems - Volume 2},
pages = {3509–3517},
numpages = {9},
location = {Montreal, Canada},
series = {NIPS'14}
}

@book{DRL-textbook,
	author = {Bellemare, Marc G. and Dabney, Will and Rowland, Mark},
	doi = {10.7551/mitpress/14207.001.0001},
	eprint = {https://direct.mit.edu/book-pdf/2111075/book\_9780262374026.pdf},
	isbn = {9780262374026},
	month = {05},
	publisher = {The MIT Press},
	title = {Distributional Reinforcement Learning},
	url = {https://doi.org/10.7551/mitpress/14207.001.0001},
	year = {2023}}

@article{FQF,
  title={Fully parameterized quantile function for distributional reinforcement learning},
  author={Yang, Derek and Zhao, Li and Lin, Zichuan and Qin, Tao and Bian, Jiang and Liu, Tie-Yan},
  journal={Advances in neural information processing systems},
  volume={32},
  year={2019}
}

@InProceedings{iqn,
  title = 	 {Implicit Quantile Networks for Distributional Reinforcement Learning},
  author =       {Dabney, Will and Ostrovski, Georg and Silver, David and Munos, Remi},
  booktitle = 	 {Proceedings of the 35th International Conference on Machine Learning},
  pages = 	 {1096--1105},
  year = 	 {2018},
  editor = 	 {Dy, Jennifer and Krause, Andreas},
  volume = 	 {80},
  series = 	 {Proceedings of Machine Learning Research},
  month = 	 {10--15 Jul},
  publisher =    {PMLR},
  url = 	 {https://proceedings.mlr.press/v80/dabney18a.html}
}

@inproceedings{morimura,
author = {Morimura, Tetsuro and Sugiyama, Masashi and Kashima, Hisashi and Hachiya, Hirotaka and Tanaka, Toshiyuki},
title = {Nonparametric return distribution approximation for reinforcement learning},
year = {2010},
isbn = {9781605589077},
publisher = {Omnipress},
address = {Madison, WI, USA},
booktitle = {Proceedings of the 27th International Conference on International Conference on Machine Learning},
pages = {799–806},
numpages = {8},
location = {Haifa, Israel},
series = {ICML'10}
}

@inproceedings{multidimensional-reward,
author = {Zhang, Pushi and Chen, Xiaoyu and Zhao, Li and Xiong, Wei and Qin, Tao and Liu, Tie-Yan},
title = {Distributional reinforcement learning for multi-dimensional reward functions},
year = {2021},
isbn = {9781713845393},
publisher = {Curran Associates Inc.},
address = {Red Hook, NY, USA},
booktitle = {Proceedings of the 35th International Conference on Neural Information Processing Systems},
articleno = {117},
numpages = {11},
series = {NIPS '21}
}

@inproceedings{mv-mdp,
author = {Mannor, Shie and Tsitsiklis, John N.},
title = {Mean-variance optimization in {M}arkov decision processes},
year = {2011},
isbn = {9781450306195},
publisher = {Omnipress},
address = {Madison, WI, USA},
booktitle = {Proceedings of the 28th International Conference on International Conference on Machine Learning},
pages = {177–184},
numpages = {8},
location = {Bellevue, Washington, USA},
series = {ICML'11}
}

@inproceedings{qr-dqn,
author = {Dabney, Will and Rowland, Mark and Bellemare, Marc G. and Munos, R\'{e}mi},
title = {Distributional reinforcement learning with quantile regression},
year = {2018},
isbn = {978-1-57735-800-8},
publisher = {AAAI Press},
booktitle = {Proceedings of the Thirty-Second AAAI Conference on Artificial Intelligence and Thirtieth Innovative Applications of Artificial Intelligence Conference and Eighth AAAI Symposium on Educational Advances in Artificial Intelligence},
articleno = {353},
numpages = {10},
location = {New Orleans, Louisiana, USA},
series = {AAAI'18/IAAI'18/EAAI'18}
}

@article{sobel,
 ISSN = {00219002},
 URL = {http://www.jstor.org/stable/3213832},
 author = {Matthew J. Sobel},
 journal = {Journal of Applied Probability},
 number = {4},
 pages = {794--802},
 publisher = {Applied Probability Trust},
 title = {The Variance of Discounted {M}arkov Decision Processes},
 urldate = {2025-04-15},
 volume = {19},
 year = {1982}
}

@InProceedings{tamar2,
  title = 	 {Temporal Difference Methods for the Variance of the Reward To Go},
  author = 	 {Tamar, Aviv and Di Castro, Dotan and Mannor, Shie},
  booktitle = 	 {Proceedings of the 30th International Conference on Machine Learning},
  pages = 	 {495--503},
  year = 	 {2013},
  editor = 	 {Dasgupta, Sanjoy and McAllester, David},
  volume = 	 {28},
  series = 	 {Proceedings of Machine Learning Research},
  address = 	 {Atlanta, Georgia, USA},
  month = 	 {17--19 Jun},
  publisher =    {PMLR},
  url = 	 {https://proceedings.mlr.press/v28/tamar13.html}
}

@inproceedings{morimura-parametric, author = {Morimura, Tetsuro and Sugiyama, Masashi and Kashima, Hisashi and Hachiya, Hirotaka and Tanaka, Toshiyuki}, title = {Parametric return density estimation for reinforcement learning}, year = {2010}, isbn = {9780974903965}, publisher = {AUAI Press}, address = {Arlington, Virginia, USA}, booktitle = {Proceedings of the Twenty-Sixth Conference on Uncertainty in Artificial Intelligence}, pages = {368–375}, numpages = {8}, location = {Catalina Island, CA}, series = {UAI'10} }

@InProceedings{bellman-gan,
  title = 	 {Distributional Multivariate Policy Evaluation and Exploration with the {B}ellman {GAN}},
  author =       {Freirich, Dror and Shimkin, Tzahi and Meir, Ron and Tamar, Aviv},
  booktitle = 	 {Proceedings of the 36th International Conference on Machine Learning},
  pages = 	 {1983--1992},
  year = 	 {2019},
  editor = 	 {Chaudhuri, Kamalika and Salakhutdinov, Ruslan},
  volume = 	 {97},
  series = 	 {Proceedings of Machine Learning Research},
  month = 	 {09--15 Jun},
  publisher =    {PMLR},
  url = 	 {https://proceedings.mlr.press/v97/freirich19a.html}
}

@article{Markowitz,
 ISSN = {00221082, 15406261},
 URL = {http://www.jstor.org/stable/2975974},
 author = {Harry Markowitz},
 journal = {The Journal of Finance},
 number = {1},
 pages = {77--91},
 publisher = {[American Finance Association, Wiley]},
 title = {Portfolio Selection},
 urldate = {2025-09-20},
 volume = {7},
 year = {1952}
}

@article{wiltzer2024action,
  title={Action gaps and advantages in continuous-time distributional reinforcement learning},
  author={Wiltzer, Harley and Bellemare, Marc and Meger, David and Shafto, Patrick and Jhaveri, Yash},
  journal={Advances in Neural Information Processing Systems},
  volume={37},
  pages={47815--47848},
  year={2024}
}

@article{wiltzer2024foundations,
  title={Foundations of multivariate distributional reinforcement learning},
  author={Wiltzer, Harley and Farebrother, Jesse and Gretton, Arthur and Rowland, Mark},
  journal={Advances in Neural Information Processing Systems},
  volume={37},
  pages={101297--101336},
  year={2024}
}

@article{artzner1999coherent,
  title={Coherent measures of risk},
  author={Artzner, Philippe and Delbaen, Freddy and Eber, Jean-Marc and Heath, David},
  journal={Mathematical finance},
  volume={9},
  number={3},
  pages={203--228},
  year={1999},
  publisher={Wiley Online Library}
}

@article{rockafellar2000optimization,
  title={Optimization of conditional value-at-risk},
  author={Rockafellar, R Tyrrell and Uryasev, Stanislav and others},
  journal={Journal of risk},
  volume={2},
  pages={21--42},
  year={2000}
}

@InProceedings{oberst2019counterfactual,
  title = 	 {Counterfactual Off-Policy Evaluation with {G}umbel-Max Structural Causal Models},
  author =       {Oberst, Michael and Sontag, David},
  booktitle = 	 {Proceedings of the 36th International Conference on Machine Learning},
  pages = 	 {4881--4890},
  year = 	 {2019},
  editor = 	 {Chaudhuri, Kamalika and Salakhutdinov, Ruslan},
  volume = 	 {97},
  series = 	 {Proceedings of Machine Learning Research},
  month = 	 {09--15 Jun},
  publisher =    {PMLR},
  url = 	 {https://proceedings.mlr.press/v97/oberst19a.html}
}

@inproceedings{buesing2019woulda,
  author       = {Lars Buesing and
                  Theophane Weber and
                  Yori Zwols and
                  Nicolas Heess and
                  S{\'{e}}bastien Racani{\`{e}}re and
                  Arthur Guez and
                  Jean{-}Baptiste Lespiau},
  title        = {Woulda, Coulda, Shoulda: Counterfactually-Guided Policy Search},
  booktitle    = {7th International Conference on Learning Representations, {ICLR} 2019,
                  New Orleans, LA, USA, May 6-9, 2019},
  publisher    = {OpenReview.net},
  year         = {2019},
  url          = {https://openreview.net/forum?id=BJG0voC9YQ},
  bibsource    = {dblp computer science bibliography, https://dblp.org}
}

@article{lorberbom2021learning,
  title={Learning generalized {G}umbel-max causal mechanisms},
  author={Lorberbom, Guy and Johnson, Daniel D and Maddison, Chris J and Tarlow, Daniel and Hazan, Tamir},
  journal={Advances in Neural Information Processing Systems},
  volume={34},
  pages={26792--26803},
  year={2021}
}

@misc{haugh2023counterfactualanalysisdynamiclatent,
      title={Counterfactual Analysis in Dynamic Latent State Models}, 
      author={Martin Haugh and Raghav Singal},
      year={2023},
      eprint={2205.13832},
      archivePrefix={arXiv},
      primaryClass={cs.LG},
      url={https://arxiv.org/abs/2205.13832}, 
}

@article{lally2025robust,
  title={Robust counterfactual inference in {M}arkov decision processes},
  author={Lally, Jessica and Kazemi, Milad and Paoletti, Nicola},
  journal={arXiv preprint arXiv:2502.13731},
  year={2025}
}

@inproceedings{
raghavan2025counterfactual,
title={Counterfactual Realizability},
author={Arvind Raghavan and Elias Bareinboim},
booktitle={The Thirteenth International Conference on Learning Representations},
year={2025},
url={https://openreview.net/forum?id=uuriavczkL}
}

@article{SZEKELY20131249,
	author = {G{\'a}bor J. Sz{\'e}kely and Maria L. Rizzo},
	journal = {Journal of Statistical Planning and Inference},
	number = {8},
	pages = {1249-1272},
	title = {Energy statistics: A class of statistics based on distances},
	volume = {143},
	year = {2013}}

@misc{kingma2017adammethodstochasticoptimization,
      title={Adam: A Method for Stochastic Optimization}, 
      author={Diederik P. Kingma and Jimmy Ba},
      year={2017},
      eprint={1412.6980},
      archivePrefix={arXiv},
      primaryClass={cs.LG},
      url={https://arxiv.org/abs/1412.6980}, 
}
\bibliographystyle{tmlr}

\end{document}